\documentclass[twoside]{article}
\usepackage[accepted]{aistats2017}

\usepackage[utf8]{inputenc} 
\usepackage[T1]{fontenc}    
\usepackage{hyperref}       
\usepackage{url}            
\usepackage{booktabs}       
\usepackage{amsfonts}       
\usepackage{nicefrac}       
\usepackage{microtype}      
\usepackage{algorithm,algorithmic}

\usepackage[round]{natbib}

\newlength{\myleftmargin}
\usepackage{amsmath}
\usepackage{amssymb}
\usepackage{amsthm}
\usepackage{color}
\DeclareSymbolFontAlphabet{\Bbb}{AMSb}

\DeclareSymbolFont{wideparensymbol}{OMX}{yhex}{m}{n}
\DeclareMathAccent{\wideparen}{\mathord}{wideparensymbol}{"F3}

\newtheorem{theorem}{Theorem}[section]

\newlength{\fixboxwidth}
\definecolor{darkgreen}{rgb}{0,0.6,0}

\newcommand{\Leq}{\,\leq \,}

\newcommand{\R}{\mathbb{R}}

\newcommand{\RR}{\mathcal{R}}

\newcommand{\D}{\mathrm{D}}

\newcommand{\fcl}{\wideparen{f}}

\newcommand{\cl}[1]{\wideparen{#1}}

\newcommand{\RLP}{\mathcal{R}_{L,P}}

\newcommand{\RB}{\mathcal{R}_{L,P}^*}

\newcommand{\bs}{\boldsymbol}
\newcommand{\TrP}[1]{{\P_X}_{|#1}}

\renewcommand{\a}{\alpha}
\renewcommand{\b}{\beta}
\newcommand{\g}{\gamma}

\newcommand{\e}{\varepsilon}

\newcommand{\n}{\eta}
\newcommand{\vt}{\vartheta}

\renewcommand{\k}{\kappa}
\newcommand{\lb}{\lambda}

\renewcommand{\t}{\tau}
\renewcommand{\P}{\mathrm{P}}

\DeclareMathOperator{\sign}{sign}

\newcommand{\eins}{\boldsymbol{1}}

\usepackage{calc}
\usepackage{graphicx}

\let\ifarxiv\iffalse
\let\ifNOTarxiv\iftrue

\begin{document}

%

%

\twocolumn[

\aistatstitle{Spatial Decompositions for Large Scale SVMs}

\ifnum\statePaper=0
\aistatsauthor{ Anonymous Author 1 \And Anonymous Author 2 \And Anonymous Author 3 }

\aistatsaddress{ Unknown Institution 1 \And Unknown Institution 2 \And Unknown Institution 3 }
\else

\aistatsauthor{ Philipp Thomann \And Ingrid Blaschzyk \And Mona Meister \And Ingo Steinwart }

\aistatsaddress{ Univeristy of Stuttgart \And University of Stuttgart \And Robert Bosch GmbH \And University of Stuttgart }
\fi ]

\begin{abstract}
  Although support vector machines (SVMs) are theoretically well understood, 
  their underlying optimization problem becomes very expensive  if, for example,
   hundreds of thousands of samples and a   non-linear kernel are considered.
  Several approaches have been proposed in the past to address this serious limitation.
  In this work we investigate a  decomposition strategy that learns on small, spatially defined  data chunks.  
  Our contributions are two fold: On the theoretical side we 
  establish an oracle inequality for the overall learning method using the hinge loss, and 
  show that the resulting rates match those known for SVMs solving the complete 
  optimization problem with Gaussian kernels. 
  On the practical side we compare our approach to learning SVMs 
  on small, randomly chosen chunks. Here it turns out that for comparable training times
  our approach is significantly faster during testing and also reduces the test error in most cases significantly.
  Furthermore, we show that our  approach easily scales up to
  10 million training samples: including hyper-parameter selection using 
  cross validation, the entire training only takes a few hours on a single machine.
  Finally, we report an experiment on 32 million training samples.
  All experiments used \texttt{liquidSVM} \citep{liquidSVM}.%
%
\end{abstract}

\section{Introduction}

Kernel methods are thoroughly understood from a theoretical perspective, 
used in many settings, and are known to give good performance for small and medium sized  
training sets. Yet they suffer from their computational complexity that grows quadratically in space and at least quadratically in 
time. For example, storing the 
entire kernel matrix to avoid costly recomputations in 
e.g.\ 64GB of memory, one can consider at most 100\;000 data points. In the last 15 years there has been wide research to circumvent this barrier
(cf.~\cite{BoChDCWe07} for an overview):
Sequential Minimal Optimization \cite{Platt99a}
allows for caching kernel rows so that 
the memory barrier is lifted,
parallel solvers
try to leverage from 
recent advances in hardware design,
matrix approximations \cite{WiSe01a}
replace the original $n\times n$ kernel matrix,
where $n$ is the number of samples, by a smaller approximation,
the random Fourier feature method \cite{RaRe08a}
approximates the kernel function by another kernel having an explicit, low-dimensional feature map, which, in a second step, makes it 
possible to solve the primal problem instead of the usually considered dual problem, see 
\cite{SrSz15a} for a recent theoretical investigation establishing optimal rates for this approximation.
Moreover, iterative  strategies \cite{RoVi15a-pub,LiRoZh15a} modify the underlying regularization approach, e.g.\ by controlled  early stopping. Finally,  various decomposition strategies have been proposed, which, in a nutshell, solve many 
small rather than one large optimization problem.

In this work we also  focus on a data decomposition strategy.
Such strategies have been investigated at least since \cite{BoVa92a,VaBo93a}. The most simple such 
strategy, called  \emph{random chunks}, 
splits the data by random into chunks of some given size, train for each chunk,
and finally average  the  decision  functions obtained on every chunk to one final decision function.
Recently, this approach has been investigated theoretically in \cite{ZhDuWa15a}.

Obviously, 
the drawback with this approach is that the test sample has to be evaluated
on every chunk: for example, if every chunk uses  80\% of its training samples as support vectors,
then the test sample  has to be evaluated using  80\% of the entire training set.
From this perspective, a more interesting strategy is to decompose the data set into spatially defined cells,
since in this case every test sample has only to be evaluated using the cell it belongs to. In our example above, this
 amounts to 80\% of the cell size, instead of 80\% of the entire training set size.
Clearly, the difference of both costs can be very significant if, for example the cell size ranges 
in between say a few thousands, but the training set contains millions of examples.
The natural next question is of course, whether and by how much one suffers from this approach in terms
of test accuracy.

Spatial decompositions for SVMs can be obtained in many different ways, e.g., by 
clusters \cite{ChTaJi07a,ChTaJi10a},
decision trees \cite{BeBl98a, WuBeCrST99a, ChGuLiLu10a}, and
$k$-nearest neighbors \cite{ZhBeMaMa06a}. 
Most of these strategies are investigated experimentally,
yet only few theoretical results are known:
\cite{Hable13a} proves universal consistency of localized versions of SVMs.
Oracle inequalities and optimal learning rates have been shown in \cite{MeSt16a}
for least squares SVMs that are trained on disjunct
spatially defined cells.

In the theoretical part of this work we expand the results in \cite{MeSt16a} to the hinge loss.
Besides the obviously  different treatment of the approximation error, the main difference in our work is that 
the least squares loss allows to use an optimal variance bound, whereas for the hinge loss such an 
optimal variance bound is only available for distributions satisfying the best version of Tsybakov's noise condition, see 
\eqref{eq.noise.exponent}.
In general, however, only a weaker variance bound is possible, which in turn makes the technical treatment harder and, surprisingly, 
the conditions on the cell radii that guarantee the best known rates, more restrictive.


In the experimental part we use \texttt{liquidSVM} \citep{liquidSVM}
to train local SVMs on some well-known data sets
to demonstrate that they provide an efficient way to tackle large-scale problems with millions of samples.
As we use full 5-fold cross validation on a $10\times10$ hyper-parameter grid
this shows that SVMs promise to achieve fully automatic machine learning.
It takes only 2-6 hours to train on data sets with 10 millions of samples
and 28 features, using only 10-30GB of memory and a single computer.
Finally, we distributed our software onto 11 machines
to attack a data set with 30 million samples and 631 features.

In Section~\ref{sec.locSVM} we give a description of local SVMs. In Section~\ref{sec.theory} we define them for the hinge loss and Gaussian kernels and state the theoretical results.
To motivate this we give in Section~\ref{sec.toy} a toy example.
In Section~\ref{sec.experiments} we describe the extensive experiments we performed.
The proofs and more details are in the supplement.

\section{The local SVM approach}
\label{sec.locSVM}

In this section we briefly describe the local SVM approach. To this end, let $D:=((x_1,y_1),\ldots,(x_n,y_n))$ be a data set of length $n$, where $x_i \in \R^d$ and $y_i \in \R$. Local SVMs construct a function $f_{D}$ by solving SVMs on many spatially defined small chunks. To be more precise, let  $(A_{j})_{j=1,\ldots,m}$ be an arbitrary partition of the input space $X$. We define for every  $j\in\{1,\ldots,m\}$ the local data set $D_j$ by
\begin{align*}
 D_{j}:=\{(x,y)\in D : x \in A_j\}.
\end{align*}
Then, one learns an individual SVM on \emph{each} cell by solving for a regularization parameter $\lb>0$ the optimization problem
\begin{align*}
 f_{D_{j},\lb}=\underset{f \in H_j}{\arg\min}\, \lb \|f\|_{H_j}^2 + \frac{1}{n}\sum_{x_i,y_i \in D_j}L(y_i,f(x_i)) 
\end{align*}
for every $j\in\{1,\ldots,m\}$, where $H_j$ is a reproducing kernel Hilbert space over $A_j$ with reproducing kernel $k_j: A_j \times A_j \to \R$, see \citet[Chapter~4]{StCh08}, and where $L: Y \times \R \to [0,\infty)$ is a function describing our learning goal. 
The final decision function  $f_{D,\lb} : X \to \R$ is then defined by
\begin{align*} 
 f_{D,\lb}(x) := \sum_{j=1}^{m} \eins_{A_{j}}(x)f_{D_{j},\lb}(x).
\end{align*}
To illustrate the advantages of this approach, 
let us assume that the size of the  data sets $D_j$ is nearly the same for all $j\in\{1,\ldots,m\}$, that means $|D_j|\approx n/m$. For example, if the data is uniformly distributed, it is not hard to show that the latter assumption holds, and if the data lies uniformly on a 
low-dimensional manifold, the same is true, since empty cells can be completely ignored. 
Now, the calculation of the kernel matrices $K_j^{i,l}=k_j(x_i,x_l)$ for $x_i,y_l \in D_j$ scales as
\begin{align*}
 O\left(m \cdot \left(\frac{n}{m}\right)^2 \cdot d\right)=O\left(\frac{n^2 \cdot d}{m} \right).
\end{align*}
In comparison, for a global SVM the calculation of the kernel matrix $K^{i,l}=k(x_i,x_l)$ scales as
\begin{align*}
O\left(n^2 \cdot d\right),
\end{align*}
such that splitting and multi-thread improve that scaling by $1/m$. 
Similarly, it is well-known that  the time complexity of the solver is 
\begin{gather}
  O\left(m \cdot \left(\frac{n}{m}\right)^{1+a}\right)= O\left(n \cdot \left(\frac{n}{m}\right)^a\right), 
  \label{eq:complexity}
\end{gather}
where  $a\in[1,2]$ is a constant,
and the test time scales as the kernel
calculation (see Table~\ref{time-table-O} in the supplement for experimental corroboration and \citet[Theorem~6]{StHuSc11a} for theoretical bounds)%
. 
In all three phases we thus see a clear improvement over a globally trained SVM.
Moreover, while SVMs trained on random chunks have the same complexities for the kernel 
matrix and the solver, they only have the bad complexity of 
the global approach during testing.


\section{Oracle inequality and learning rates for local SVMs with hinge loss}
\label{sec.theory}

The aim of this section is to  theoretically investigate
the local SVM approach for binary classification. To this end, we define for a measurable function $L: Y\times\R\to[0,\infty)$, called loss function, the $L$-risk of a measurable function $f:X\to\R$ by
\begin{align*}
 \RR_{L,\P}(f)=\int_{X\times Y} L(y,f(x)) \,d\P(x,y).
\end{align*}
Moreover, we define the optimal $L$-risk, called the Bayes risk with respect to $\P$ and $L$, by
\begin{align*}
 \RR^*_{L,\P}:=\inf\left\{\RR_{L,\P}\left(f\right) \ |
 \ f : X\to \R \text{ measurable}\right\} \,
\end{align*}
and call a function $f^*_{L,\P} : X\to\R$, attaining the infimum, Bayes decision function. Given a data set $D:=((x_1,y_1),\ldots,(x_n,y_n))$ sampled i.i.d.\ from a probability measure $P$ on $X \times Y$, where $X \subset \mathbb{R}^d$ and $Y:=\lbrace -1,1 \rbrace$, the learning target in binary classification is to find a decision function $f_D \colon X \to  \mathbb{R}$ such that $\sign f_D(x)=y$  for new data $(x,y)$ with high probability. A loss function describing our learning goal is the classification loss $L_{\text{class}}: Y\times\R\to[0,\infty)$, defined by
\begin{align*}
 L_{\text{class}}(y,t):=\mathbf{1}_{(0,\infty]}(y\cdot \sign t),
\end{align*}
where $\text{sign } 0:=1$. Another possible loss function is the hinge loss $L_{\text{hinge}}: Y\times\R\to[0,\infty)$, defined by
\begin{align*}
 L_{\text{hinge}}(y,t):=\max \lbrace 0,1-yt \rbrace\,
\end{align*}
for $y=\pm 1,\, t \in \R$, which is even convex.
Since a well-known result by Zhang, e.g.\ \citet[Theorem~2.31]{StCh08}, shows that
\begin{align*}
\mathcal{R}_{L_{\text{class}},P}(f) - \mathcal{R}_{L_{\text{class},P}}^*   \leq \mathcal{R}_{ L_{\text{hinge}},P}(f)-\mathcal{R}_{L_{\text{hinge}},P}^*
\end{align*}
for all functions $f:X \to \mathbb{R}$, we consider in the following the hinge loss in our theory and write $L:= L_{\text{hinge}}$. We remark that for the hinge loss it suffices to consider the risk for function values restricted to the interval $[-1,1]$, since this does not worsen the loss and thus the risk. Therefore, we define by 
\begin{align*}
\cl{t}:=\max \lbrace -1, \min\lbrace t,1 \rbrace\rbrace
\end{align*}
for $t \in \R$ the clipping operator, which restricts values of $t$ to $[-1,1]$, see \citet[Chapter~2.2]{StCh08}. 
In order to derive a bound on the excess risks $\RLP(\fcl_{D,\lb})-\RB$, and in order to derive learning rates, we recall some notions from \citet[Chapter~8]{StCh08}, which describe the behaviour of $P$ in the vicinity of the decision boundary. To this end, let $\eta \colon X \to [0,1]$ be a version of the posterior probability of $P$, that means that the probability measures $P(\,\cdot \,|x)$, defined by $P(y=1|x)=:\eta(x),\,x\in X$, form a regular conditional probability of $P$. We write
\begin{align*}
X_{1} &:= \lbrace\,x \in X \colon \n(x)>1/2 \,\rbrace,\\
X_{-1} &:=  \lbrace\, x \in X \colon \n(x)<1/2 \,\rbrace.
\end{align*}
Then, we call the function $\Delta_{\n} \colon X \to [0,\infty)$, defined by
\begin{align*}
\begin{split}
\Delta_{\n}(x):=\begin{cases}
  \text{dist}(x,X_1)  & \text{if}\, x \in X_{-1},\\
  \text{dist}(x,X_{-1})  & \text{if}\, x \in X_{1},\\
  0 & \text{otherwise},
\end{cases}
\end{split}
\end{align*}
distance to the decision boundary, where $\text{dist}(x,A):=\inf_{x' \in A}\|x-x'\|_2$. Thus we can describe the mass of the marginal distribution $P_X$ of $P$ around the decision boundary by the following exponents: 
We say that $P$ has margin-noise exponent (MNE) $\b \in (0,\infty)$ if there exists a constant $c_{\text{MNE}} \geq 1$ such that
\begin{align*}
\int_{\lbrace \Delta_{\n}(x)<t \rbrace} |2\n(x)-1| \,dP_X(x) \leq (c_{\text{MNE}}t)^{\beta}
\end{align*}
for all $t>0$. That is, the MNE $\b$ measures the mass and the noise, i.e.\ points $x \in X$ with $\n(x) \approx 1/2$, around the decision boundary. Hence, we have a large 
MNE if we have low mass and/or high noise around the decision boundary. Furthermore, we say that $P$ has noise exponent (NE) $q \in [0, \infty ]$ if there exists a constant $c_{\text{NE}}>0$ such that  
\begin{align}\label{eq.noise.exponent}
P_X(\lbrace x\in X: |2\n(x) -1|<\e \rbrace ) \leq (c \e)^q 
\end{align}
for all $\e > 0$. Thus, the NE $q$, which corresponds to Tsybakov's noise condition, introduced in \citet{Tsybakov04a}, measures the amount of noise, but does not locate it. For examples of typical values of these exponents and relations between them we refer the reader to \citet[Chapter~8]{StCh08}.

For the theory of local SVMs it is necessary to specify the partition. Hence, we assume in the following that $X \subset [-1,1]^d$ and define for a set $T \subset X$ its radius by
\begin{align*}
r_T= \inf \lbrace \varepsilon >0: \exists \ s \in T \,\text{such that}\, T \subset B_{2}(s,\varepsilon)  \rbrace,
\end{align*}
where $B_{2}(s,\varepsilon):=\lbrace t \in T: \|t-s\|_2 \leq \varepsilon \rbrace$ with Euclidean norm $\| \cdot \|_2$ in $\mathbb{R}^d$. In the following let $(A_{j})_{j=1,\ldots,m}$ be a partition of $X$ such that all its cells have non-empty interior, that is $\mathring{A}_j \neq \emptyset$ for every $j \in \lbrace 1,\ldots,m \rbrace$, and such that for $r>0$ we have 
\begin{align}\label{ex. Ueberdeckung}
 r_{A_j} < r \leq 16 m^{-\frac{1}{d}}.
\end{align}
A simple partition that fulfills the latter condition is for example the partition by cubes with a specific side length. We refer the reader to Section \ref{sec.toy} for the computation of another
type of partition satisfying (\ref{ex. Ueberdeckung}).

In the following we restrict ourselves to local SVMs with Gaussian kernels. For this purpose, we denote for every $j \in \lbrace 1,\ldots,m \rbrace$ by $H_{\g_j}(A_j)$ the RKHS over $A_{j}$ with Gaussian kernel $k_{\g_j} : A_{j}\times A_{j}\to\R$, defined by
\begin{align*}
 k_{\g_j}(x,x'):=\exp\left(-\g_j^{-2}\|x-x'\|_2^2\right) 
\end{align*}
for some width $\g_j>0$. Furthermore, we define for $f\in H_{\g_j}(A_j)$ the function $\hat{f} : X\to\R$ by
\begin{align*}
 \hat{f}(x) := \begin{cases}
                f(x)\,, & x\in A_j\,, \\
                0\,, & x\notin A_j\,.
               \end{cases}
\end{align*}
Then, according to \citet[Lemma~2]{EbStXXa} the space $\hat{H}_{\g_j}:=\{\hat{f} : f\in H_{\g_j}(A_j)\}$ equipped with the norm 
\begin{align*}
 \|\hat{f}\|_{\hat{H}_{\g_j}}:=\|f\|_{H_{\g_j}(A_j)} 
\end{align*}
is an RKHS over $X$. Thus, local SVMs for Gaussian kernels solve  the optimization problem 
\begin{align*}
 f_{D_{j},\lb_j,\g_j}=\underset{\hat{f}\in \hat{H}_{\g_j}}{\arg\min} \,\lb_j \|\hat{f}\|_{\hat{H}_{\g_j}}^2
 + \frac{1}{n}\sum_{x_i,y_i \in D_j}L(y_i,f(x_i)),
\end{align*}
for every $j \in \lbrace 1,\ldots,m \rbrace$, where $\lb_j,\g_j>0$. Then, for the vectors $\bs\g:=(\g_1,\ldots,\g_m)$ and $\bs\lb:=(\lb_1,\ldots,\lb_m)$ the decision function $f_{D,\bs\lb,\bs\g} : X \to \R$ is defined by
\begin{align} \label{VP-SVM}
 f_{D,\bs\lb,\bs\g}(x) := \sum_{j=1}^{m} f_{D_{j},\lb_j,\g_j}(x).
\end{align}
Note that the clipped decision function $\cl{f}_{D,\bs\lb,\bs\g} : X \to [-1,1]$ is then defined by the sum of the clipped empirical solutions $\cl{f}_{D_{j},\lb_j,\g_j}$ since for all $x \in X$ there is exactly one $f_{D_{j},\lb_j,\g_j}$ with $f_{D_{j},\lb_j,\g_j}(x) \neq 0$. We finally 
introduce the following set of assumptions:
\begin{itemize}
 \item[\textbf{(A)}]
Let $(A_{j})_{j=1,\ldots,m}$ be a partition of $X$ and $r>0$, 
such that for every $j \in \lbrace 1,\ldots,m \rbrace$ we have $\mathring{A}_j \neq \emptyset$ and (\ref{ex. Ueberdeckung}). In addition, for every 
$j\in\{1,\ldots,m\}$ let $H_{\g_j}(A_j)$ be the RKHS of the Gaussian kernel 
$k_{\g_j}$ over $A_j$ with width $\g_j>0$. Furthermore, we use the notation $\g_{\text{max}}:=\max\lbrace \g_1, \ldots, \g_m \rbrace$.
\end{itemize}
Now, we present an upper bound on the excess risk for the hinge loss and Gaussian kernels. 
\begin{theorem}\label{corollar}
Let $Y=\lbrace{-1,1 \rbrace}$ and let $L: Y\times\R\to [0,\infty)$ be the hinge loss. Let $P$ be a distribution on $X \times Y$ 
with MNE $\b \in (0,\infty)$ and NE $q \in [0, \infty ]$. Moreover, let (A) be satisfied. 
Then, for all $p \in (0,1)$, $n \geq 1$, $\t \geq 1$ with $\t \leq n$, $\bs\lb=(\lb_1,\ldots,\lb_m)\in (0,\infty)^m$, $\bs\g=(\g_1,\ldots,\g_m)\in (0,r]^m$ the SVM given by \eqref{VP-SVM} 
satisfies 
\begin{align*}
\ifNOTarxiv&\fi
\RLP(\fcl_{D,\bs\lb,\bs\g})-\RB
\ifNOTarxiv\nonumber\\ \fi
  &\Leq  C_{\b,d,p,q} \Biggl(    \sum_{j=1}^m \lb_j \g_j^{-d}+ \g_{\text{max}}^{\b} 
    + \left(\frac{\t}{n}\right)^{\frac{q+1}{q+2}} \nonumber\\ 
  &\quad +\left( r^{2p} \left(\sum_{j=1}^m  \lb_j^{-1}  \g_j^{-\frac{d+2p}{p}} P_X(A_{j}) \right)^p n^{-1} \right)^{\frac{q+1}{q+2-p}}  \Biggr)
\end{align*}
with probability $P^n$ not less than $1-3e^{-\t}$, where the constant $C_{\b,d,p,q}>0$ depends only on $d,\b,p$ and $q$.
\end{theorem}

The main idea for the proof is that $f_{D,\bs\lb,\bs\g}$ is an SVM solution for a particular RKHS. To be more precise, it is easy to show with a generalization of \citet[Lemma~3]{EbStXXa} that for RKHSs $\hat{H}_{\g_j}$ on $X$ and a vector $\bs\lb:=(\lb_1,\ldots,\lb_m)>0$ the direct sum   
\begin{align*}
 H := \bigoplus_{j=1}^m \hat{H}_{\g_j} 
      = \Bigg\lbrace{ f=\sum_{j=1}^m f_j : f_j\in\hat{H}_{\g_j} \text{ for all } j\in J \Bigg\rbrace} 
\end{align*}
is again an RKHS if it is equipped with the norm
\begin{align*}
 \|f\|_{H}^2 = \sum_{j=1}^m \lb_j \|f_j\|_{\hat{H}_{\g_j}}^2\,.
\end{align*}
Obviously, $f_{D,\bs\lb,\bs\g} \in H$. We remark that the latter construction 
actually holds for RKHSs with arbitrary kernels. For the proof it is necessary on the one hand to chose suitable RKHSs in order to bound the approximation error and on the other hand to bound the entropy numbers of the operator $\text{id}: H_{\g_j}(A_j) \mapsto L_2(P_{X_{|A_j}})$, which describe in some sense the size of the RKHS $H_{\g_j}(A_j)$. Unfortunately, such
bounds contain constants which depend on the cells $A_j$ and which are in general hard to control. However, for Gaussian kernels we obtain such bounds if we restrict $\g_j$ by $r$ for every $j \in \lbrace 1,\dots,m \rbrace$, which explains our restriction to Gaussian kernels. Moreover, the proof 
actually
shows that $\g_j\leq r$ can be replaced by $\g_j \leq cr$ for a constant $c$, which is independent of $p,m,\t,\lb$ and $\g$. Furthermore, using this entropy number bound leads to a dependence of a parameter $p$ on the right-hand side of the oracle inequality, where small $p$ lead to an unknown behaviour in the constant $C_{\b,d,p,q}$. This problem is well-known in the Gaussian kernel case, see \citet{StSc07a} or \citet{EbSt11a}, and there is not given a solution for this problem yet.

Let us assume that all cells have the same kernel parameter $\g$ and regularization parameter $\lb$. Then, the oracle inequality stated in Theorem \ref{corollar} coincides up to constants and to the parameter $p$, which can be chosen arbitrary small, the oracle inequality for global SVMs stated in \citet[Theorem~8.25]{StCh08} for $\t=\infty$. Furthermore, we remark that our oracle inequality is formally similar to the inequality for local SVMs for the least square loss, see \citet[Theorem~7]{EbStXXa} if we assume that the Bayes decision function is contained in a Besov space with smoothness $\a=\b/2$. 

In the next theorem we show learning rates by choosing appropriate sequences of $r_{n}, \lb_{n,j}$ and $\g_{n,j}$.

\begin{theorem}\label{Thm hinge rate}
Let $\t \geq 1$ be fixed and $\nu \in \left(0, \frac{q+1}{\b(q+2)+d(q+1)} \right]$. Under the assumptions of Theorem~\ref{corollar} and with
\begin{align*}
 r_n &\,=\, c_1 n^{-\nu} \,,\\ 
 \lb_{n,j} &\,=\, c_2 r_n^d n^{-\frac{(\b+d)(q+1)}{\b(q+2)+d(q+1)}}\,, \\
 \g_{n,j} &\,=\, c_3 n^{-\frac{(q+1)}{\b(q+2)+d(q+1)}}\,,  
\end{align*}
for every $j \in \lbrace 1, \ldots, m_n \rbrace$, we have for all $n \geq 1$ and $\xi >0$ that 
\begin{align*}
\ifNOTarxiv&\fi\RLP(\fcl_{D,\bs\lb_n,\bs\g_n})-\RB \ifNOTarxiv\\\fi
 &\Leq C_{\b,\nu,\xi,d,q} \tau^{\frac{q+1}{q+2}} \cdot  n^{-\frac{\b(q+1)}{\b(q+2)+d(q+1)}+\xi}
\end{align*}
holds with probability $P^n$ not less than $1-3e^{-\t}$, where $\bs\lb_n:=(\lb_{n,i})_{i=1,\ldots,m_n}$ as well as $\bs\g_n:=(\g_{n,i})_{i=1,\ldots,m_n}$ and where $C_{\b,\nu,\xi,d,q},c_1,c_2,c_3$ are positive constants. 
\end{theorem}

Note that the restriction of the parameter $\nu$ in Theorem \ref{Thm hinge rate} is set to ensure that 
$\sup_{n \geq 1} \g_{n,j}/r_{n} < \infty$, as we mentioned after Theorem \ref{corollar}. The learning rate stated in Theorem \ref{Thm hinge rate} coincides always with 
the fastest known rate which can be achieved by a global SVM, cf. \citet[Theorem~8.26 and (8.18)]{StCh08}. Let us consider some special cases for the parameters $\b$ and $q$. In the case of "benign noise", that is $q=\infty$, our learning rate reduces to
\begin{align*}
n^{-\frac{\b}{\b+d}+\xi},
\end{align*}
and is only less sensitive to the dimension $d$ if $\b$ is large. Next, let us assume that $q=1$ such that $P$ has a rather moderate noise concentration and that $\b=2$, which means that we have additionally much mass around the decision boundary. For examples of distributions having these parameters we refer to the examples in \citet[Chapter~8]{StCh08}. The chosen parameters $q$ and $\b$
yield the rate
\begin{align*}
n^{-\frac{2}{3+d}+\xi}
\end{align*}
and we observe that the dimension $d$ has a high impact, which means that the rate gets worse the higher $d$. We refer the reader to Section~\ref{sec.toy}, where we created a toy example for a distribution having these parameters $\b$ and $q$ and where we compare this rate with the experimental one. Since the class of considered distributions contains the ones, whose marginal distribution $P_X$ is the uniform distribution, a bad dependence on $d$ is not surprisingly as these distributions usually among those which cause the curse of dimensionality. However, if the data lies for example on a $d_{\mathcal{M}}$-(low)-dimensional rectifiable manifold $\mathcal{M}$, we believe that one is able to improve the rate since one would learn the local SVM only on a few cells---instead of learning on $m=O(r^d)$ cells, one only has to learn on $m=O(r^{d_{\mathcal{M}}})$ cells.

Clearly, to obtain the rates in Theorem \ref{Thm hinge rate}
we need to know the parameters $\b$ and $q$. 
However, such rates can also be obtained by a data-dependent parameter selection strategy without knowing the parameters. For example, \citet{EbStXXa} presented for the least-square loss the training validation Voronoi partition support vector machine (TV-VP-SVM) and showed that this learning method achieves adaptively the same rates as the rates for local SVM for the least-square loss.  We remark that this method can be adapted to our case, that is for the hinge loss, since a key ingredient is an oracle inequality having the structure of \citet[Theorem~7]{EbStXXa}, which is given in our case, as we mentioned in the discussion before Theorem \ref{Thm hinge rate}.
For more details to parameter selection methods we refer the reader to \citet[Section~5]{EbStXXa} and \citet[Chapter~6.5 and Chapter~8.2]{StCh08}. At this point we finally remark that the mentioned adaptive strategies obtain the rate
\begin{align*}
n^{-\frac{\b(q+1)}{\b(q+2)+d(q+1)}+\xi}
\end{align*}
for all $\b$ and $q$ satisfying 
$\nu \leq \frac{\b(q+1)}{\b(q+2)+d(q+1)}$. Consequently, for larger $\nu$, which lead to 
faster training times, the range of adaptivity becomes smaller, and vice-versa.

\section{Toy Example}
\label{sec.toy}

To illustrate the theoretical results we now give a toy example
by mixing two multivariate Gaussians, one for $y=1$ and one for $y=-1$,
see Figure~\ref{figure:toy}.
\begin{figure}[ht]
  \centering
\ifarxiv
  \includegraphics[width=0.4\textwidth]{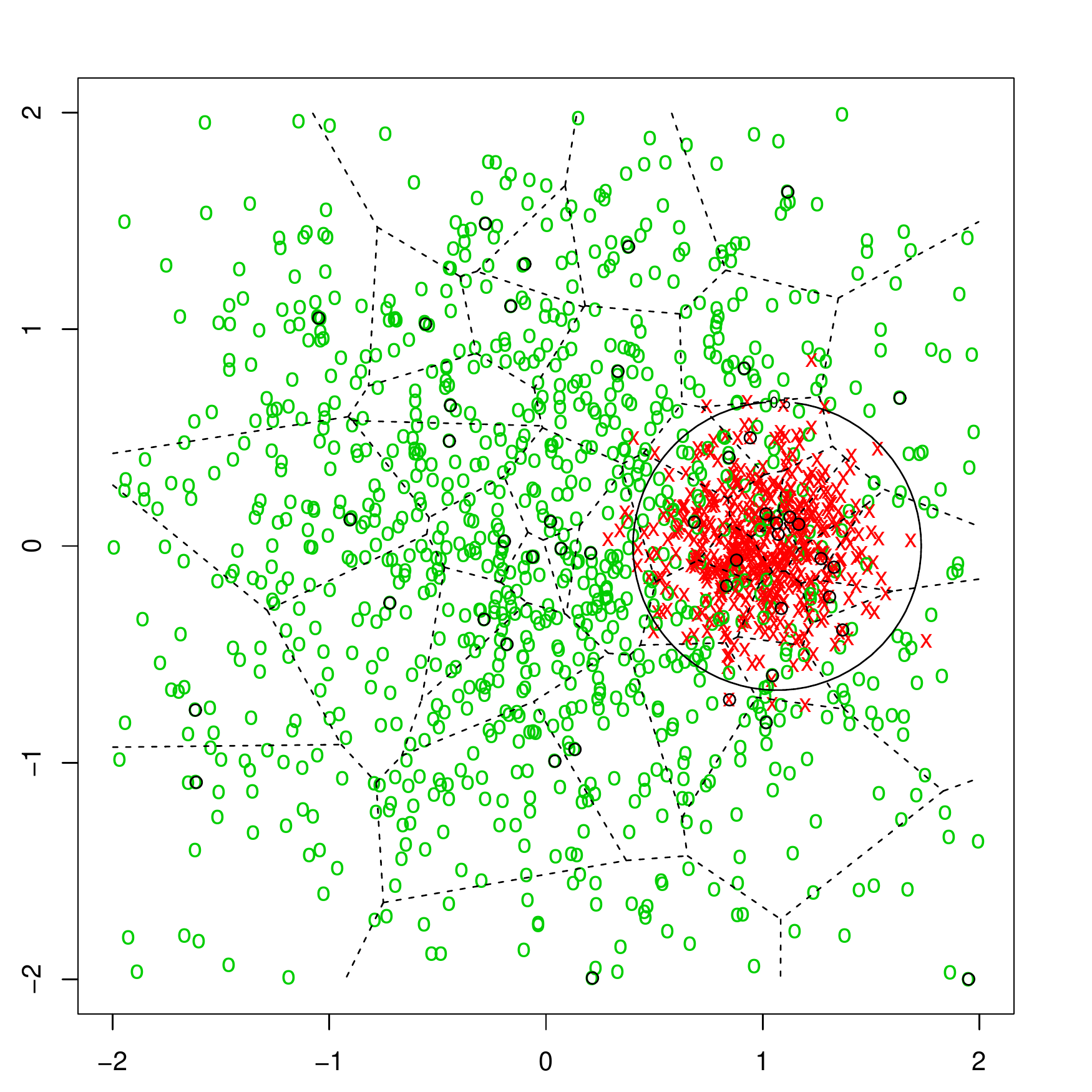}
\else
  \includegraphics[width=\textwidth/5]{plot-toy}
\fi
  \caption{Toy model: Mixture of two Gaussians.
  The decision boundary is depicted as almost a circle
  and the dotted lines mark a Voronoi partition induced
  by the black center samples.
  \label{figure:toy}}
\end{figure}
We define $X:=[-2,2]^d$ and $x_1:=(1,0,\ldots,0)\in X$
and $\vt:=0.6$. Moreover, we set for all events $A\subset X\times Y$:
\begin{align*}
P(A) &{}:= \vt\,\int_{A^+} \phi(x) \, dx
+ (1-\vt)\,\int_{A^-} \psi(x) \, dx,
\end{align*}
where $A^+:=\{x\mid (x,1)\in A\}$ and $A^-:=\{x\mid (x,-1)\in A\}$. Here $\phi$ and $\psi$ are the densities of the multivariate normal
distribution around the origin, and $x_1$ resp.\ and of variance
$1$ and $\tfrac18$ resp. on $X$.
In this case $\eta(x)$ can be calculated and it is easy
to see that the decision boundary is an ellipsoid,
the NE is $q=1$, and the MNE is $\beta=2$.
Hence, we can use Theorem~\ref{Thm hinge rate} by using
\begin{align*}
\nu &{}= \frac1{3+d}, &
r_n &{}= c_1 n^{-\frac1{3+d}}, \\
\lambda_n &{}= c_2 n^{-\frac{2+d}{3+d}}, &
\gamma_n &{}= c_3 n^{-\frac1{3+d}}.
\end{align*}
to achieve the rate $n^{-\frac2{3+d}+\e}$. A (even faster) rate than the theoretical from Theorem~\ref{Thm hinge rate}.

\begin{figure}[htb]
  \centering
\ifarxiv
  \includegraphics[width=\textwidth/2]{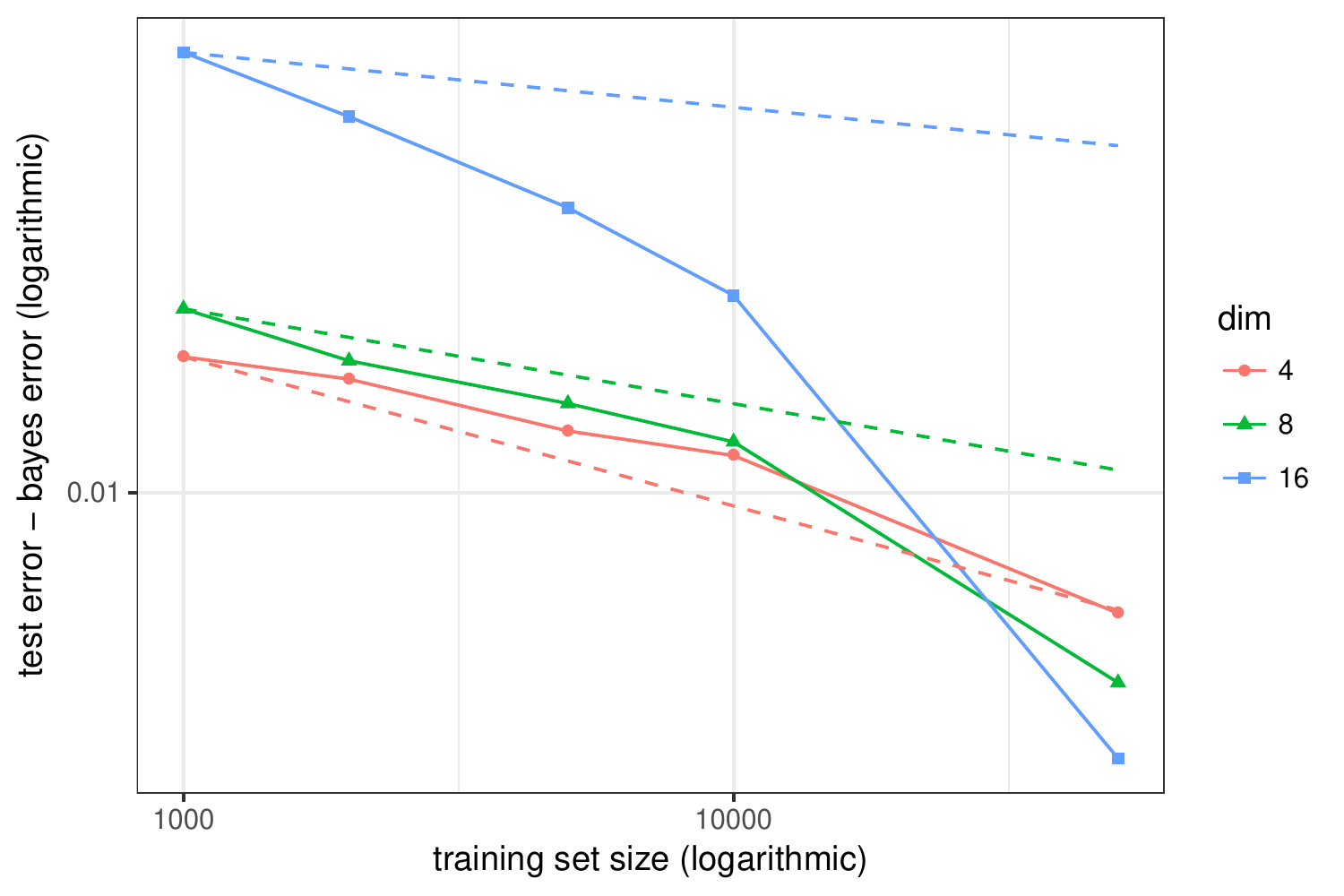}
\else
  \includegraphics[width=\textwidth/4]{plot-toy-err}
\fi
  \caption{Obtained excess risks for the toy models at dimensions $4,8,16$.
  We decrease the maximal data radius per cell as in Theorem~\ref{Thm hinge rate} (average of 20 runs).
  The rates from Theorem~\ref{Thm hinge rate} are marked in dashed lines
  with normalization to match the first data point.
  \label{figure:toy-err}}
\end{figure}

%
%

\section{Experiments}
\label{sec.experiments}

We now use the developed methods for large scale data sets
in order to understand whether the theoretical and synthetic results
also transfer to real world data.
We intend to demonstrate that partitioning
allows for large $n$ to give efficiently good results,
but only if one uses spatial decomposition.
As discussed above, global kernel methods become unfeasible
for $n\ge 100\;000$, but there is already known the decomposition method of 
\emph{random chunks}:
split the data into samples and train on these smaller sets first,
then average predictions over these samples.
We will see that spatial decompositions often outperform
random chunks in terms of test error---while testing is much faster.
We did not consider any other method for speeding up SVMs, e.g.\ random fourier features,
since these can be naturally combined with our spatial approach. Indeed, if such
a method is faster than a global SVM on data sets of, say 20.000 and more samples, then one could
also use this method on each cell of that size, which in turn would further decrease the
overall training time of our approach.
In this sense, our reported training times are a kind of worst-case scenario.

In Section~\ref{subsec:ecbdl} we will demonstrate that the spatial decomposition approach
even can be used for problems with dozens of millions of samples and hundreds of dimensions, that is too big for a single machine.

\begin{figure}[ht]
  \centering%
\ifarxiv
  \includegraphics[width=0.8\textwidth]{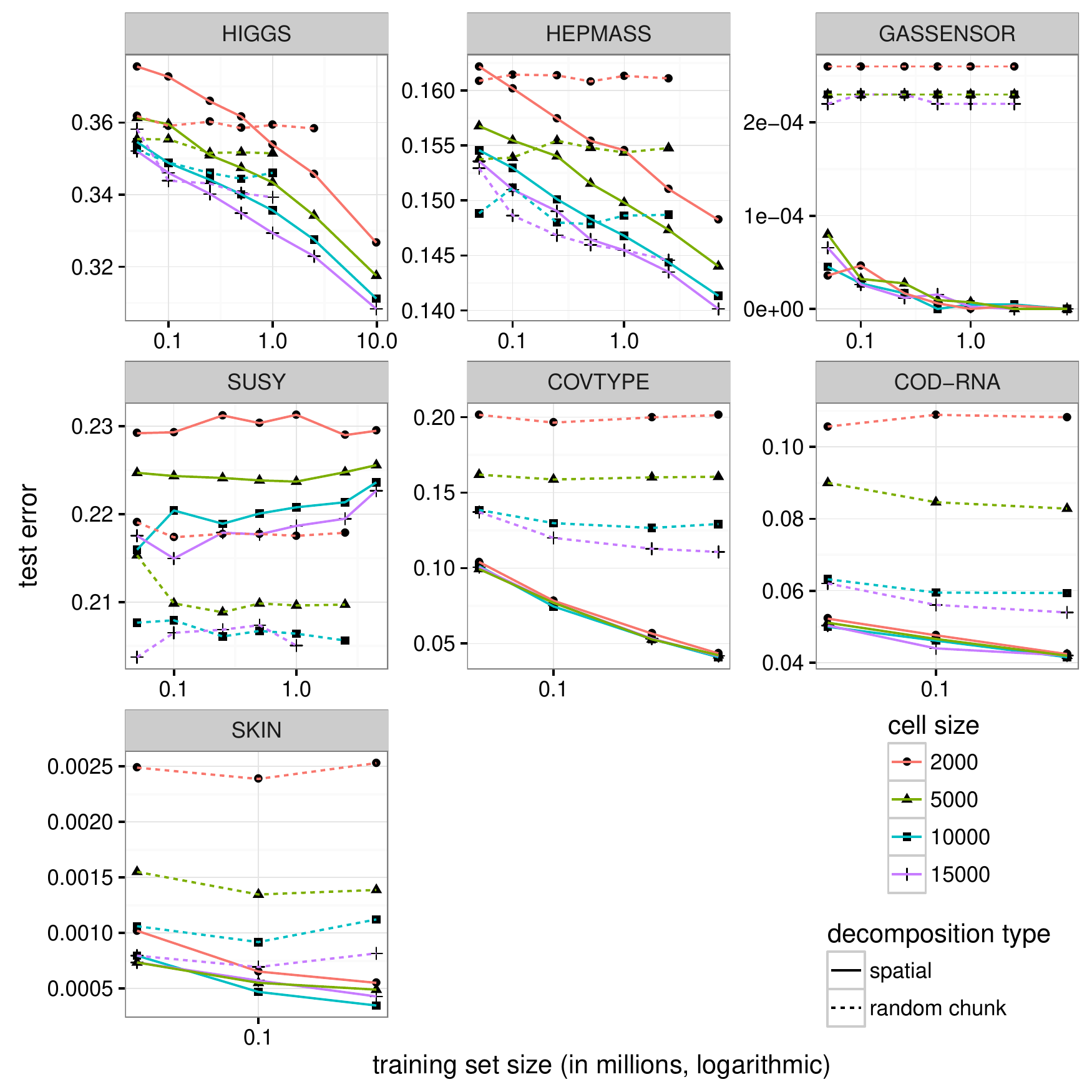}
\else
  \includegraphics[width=\textwidth/2-1ex]{plot-size-error}
\fi
  \caption{Test errors for large-scale data sets.
  The training set size is the most influential factor for
  spatial decompositions.
  Bigger cell sizes give only sometimes earlier better results.
  Random chunk decomposition only profits from bigger and not from more chunks.
  \label{figure:size-error}}
\end{figure}

\begin{table}[hbtp]
  \caption{Overview of the data sets we used.}
  \label{datasets-table}
  \centering\footnotesize
  \begin{tabular}{lrrrl}
    \toprule
    Name     & train size & test size & dims 
    & Source \\
    \midrule
\sc higgs    &  9\,899\,999& 1\,100\,001&   28 
& UCI\\
\sc hepmass  &  7\,000\,000& 3\,500\,000&   28 
& UCI\\
\sc gassensor&  7\,548\,087&  838\,678&   18 
& UCI\\
\sc susy     &  4\,499\,999&  50\,0001&   18 
& UCI\\
\sc covtype & 464\,429 & 116\,583 & 54 
& UCI\\
\sc cod-rna & 231\,805 & 99\,347 & 9 
& LIBSVM \\
\sc skin & 196\,045 & 49\,012 & 3 
& UCI \\
    \bottomrule
  \end{tabular}
\end{table}

%
%
%
%

For the data sets, we looked on the UCI and LIBSVM repositories for the
biggest examples, which did have firstly a suitable learning target,
secondly only numerical features,
and thirdly not too many dimensions since for high dimensions spatial decompositions
become at least debatable. Hence, we did not consider any image data sets.
See Table~\ref{datasets-table} for an overview of the data sets and the train/test splits.
We trained on one hand using the full training set, and on the other hand
using a sample of the full training set of
the 6 sizes $n=50\;000,\ldots,1\;000\;000,2\;500\;000$.

Even though our theoretical results are by control of radii,
in applications it is more important to control the computational cost, namely the maximal number of samples in the cells as we discussed in Section~\ref{sec.locSVM}. 
Hence, for each cell size $2000,5000,10000,15000$,
each of the 7 training sets was first split spatially into cells of that size
and a SVM was trained on each cell.
This means that on each cell five-fold cross validation was performed for a $10\times10$ geometrically spaced hyper-parameter grid where the endpoints were scaled to accommodate the number of samples in every fold, the cell size, and the dimension.
Then, for the full testing set the test error was computed
by assigning each test sample to the cell it spatially belongs to
and the cell's SVM was used to predict a label for it.
For comparison we also performed these experiments using random chunks:
Each training set was split uniformly into chunks of size $2000,5000,10000,15000$.
There, due to time constraints, for the bigger data sets we calculated the testing error using a test set of size 100\;000.

We used \texttt{liquidSVM} \citep{liquidSVM}, our own SMO-type implementation in \texttt{C++}.
As architecture, we used Intel\textsuperscript\textregistered\ Xeon\textsuperscript\textregistered\ CPUs (E5-2640 0 at 2.50GHz, May 2013) with Ubuntu Linux.
There were two NUMA-sockets each with a CPU having 6 physical cores, but multi-threading was only used to compute
the kernel matrix in training and the test error.
The data-partitioning and the solver are single-threaded.
Even though there were 128GB RAM per NUMA-socket available, the processes were limited to use at most 64GB to give results which could be compared on other workstations.
The smaller cell sizes use considerably less memory.
But this way, we restricted us to only use cell sizes up to 15\;000.
For random chunks in the biggest cases even that was impossible.

\begin{table}[hbtp]
  \caption{Time to train the local SVM including
           5-fold cross-validation on a $10\times10$ hyper-parameter grid.
}
  \label{time-table}
  \centering\footnotesize
  \begin{tabular}{lrrrr}
    \toprule
& 2000& 5000& 10000& 15000
\\
\midrule
&\multicolumn{4}{c}{training time (in min.)} \\ \cmidrule{2-5}
\sc higgs & 308 & 679 & 1358 & 1992 \\
\sc hepmass & 145 & 316 & 624 & 964 \\
\sc gassensor & 90 & 214 & 421 & 636 \\
\sc susy & 121 & 261 & 513 & 779 \\
\sc covtype & 9 & 18 & 39 & 54 \\
\sc cod-rna & 4 & 9 & 16 & 25 \\
\sc skin & 2 & 6 & 10 & 16 \\

\bottomrule
  \end{tabular}
\end{table}

\begin{table}[hbtp]
  \caption{Time to predict using random chunks divided by time to predict using spatial decomposition.
  Since RC becomes to expensive this is using 100\;000 training and test samples.
}
  \label{time-table-testFactor}
  \centering\footnotesize
  \begin{tabular}{lrrrr}
    \toprule
& 2000& 5000& 10000& 15000
\\
\midrule
\sc higgs & 74 & 34 & 15 & 10 \\
\sc hepmass & 84 & 28 & 14 & 6 \\
\sc gassensor & 667 & 513 & 254 & 414 \\
\sc susy & 88 & 35 & 19 & 11 \\
\sc covtype & 153 & 61 & 26 & 19 \\
\sc cod-rna & 139 & 52 & 18 & 6 \\
\sc skin & 25 & 10 & 8 & 4 \\

\bottomrule
  \end{tabular}
\end{table}



The results on real world data sets give a clear picture.
It can be seen in Figure~\ref{figure:size-error} that increasing the training set size
enhances the test error in most cases dramatically.
All data sets but \textsc{susy} show that by using spatial partitioning one is able to attack large-scale problems nicely.
Remark that all cell sizes give almost the same error.
Yet smaller cell sizes give it faster (see Table~\ref{time-table}), only for
\textsc{higgs} and \textsc{hepmass} bigger cells achieve the same test error
already using fewer training data.
Bigger cell sizes can give better test error,
although in the trade-off time vs.\ error, they play not too big of a role,
cf.\ Figure~\ref{figure:time-error} in the supplement.
In contrast, for the random chunks strategy the cell size plays the most crucial
role for the test error and the error quickly saturates at some value.
The error can only be made smaller by using bigger cell sizes,
and hence much more expensive training and even more testing time, see Table~\ref{time-table-testFactor}.
Finally, our spatial decomposition seems overwhelmed with \textsc{susy}
and exposes the same saturation effect as in random chunk and
does not achieve to beat that.


\begin{table*}[hbtp]
  \caption{More detailed times at fixed cell size 2000 (in seconds).
  The testing phase is given per 100\;000 test samples.
  The last column gives the maximal resident size over all of training
  and testing (in GB).}
  \label{time-fine-table}
  \centering\footnotesize
  \begin{tabular}{lrrrrrrr}
    \toprule
    name& \multicolumn{6}{c}{time (in seconds) used in phase} & RAM \\
	& part. & kernel calc.&solver &valid. &sel & test &
	\\
	\midrule
	\sc higgs & 40 & 2887 & 11862 & 994 & 1030 & 322 & 31 \\
\sc hepmass & 26 & 2022 &  4235 & 682 &  560 & 208 & 22 \\
\sc gassensor & 21 & 2354 &   729 & 633 &  541 & 166 & 21 \\
\sc susy & 18 & 1296 &  4419 & 448 &  374 & 141 & 13 \\
\sc covtype &  3 &  138 &   206 &  44 &   47 &  13 &  3 \\
\sc cod-rna &  1 &   69 &    76 &  20 &   21 &   7 &  1 \\
\sc skin &  1 &   53 &    21 &  18 &   11 &   3 &  1 \\

	\bottomrule
  \end{tabular}
\end{table*}

\subsection{Big-data: ECBDL 2014}
\label{subsec:ecbdl}

One of the most important aspects of spatial decompositions is
that this makes the process cloud scalable,
since the training and testing on the cells trivially can be assigned
to different workers, once the data is split.
To demonstrate this we used Apache Spark, an in-memory map/reduce framework.
The data set was saved on a Hadoop distributed file system on one master and up to 11 worker machines of the above type.
In a first step, the data was split into coarse cells by the following procedure.
A subset of the training data was sampled and sent to the master machine
where 1000-2000 centers were found and these centers were sent back to the worker machines.
Now each worker machine could assign locally to every of it's samples the coarse cell
in the Voronoi sense.
Finally a Spark-shuffle was performed: Every cell was assigned to one of the workers
and all its samples were sent to that worker.
This procedure is quite standard for Spark.
It had to be performed only once for the data set.

In the second step every such coarse cell--now being on one physical machine--was
used for training by our \texttt{C++} implementation discussed above:
this in particular means that each coarse cell was again split into fine cells
of some specified size and then the TV-VP-SVM method was used.
For these experiments we used a $20\times20$ hyper-parameter grid.
Obviously this now was done in parallel on all worker nodes.
The test set was also split into the coarse cells
and then by our implementation further into fine cells for prediction.

To give an experimental evaluation we used the classification data set introduced in
the \emph{Evolutionary Computation for Big Data and Big Learning Workshop}\footnote{
See \url{http://cruncher.ncl.ac.uk/bdcomp/} and
\url{http://cruncher.ncl.ac.uk/bdcomp/BDCOMP-final.pdf}.
} (ECBDL).
The data set considers contact points of polymers.
It has 631 dimensions and 32 million training (60GB disk space) and
2.9 million test samples (5.1GB disk space).
One of the major challenges with this is a rather strong imbalance in the labels--there are 98\% of negative samples.
Therefore the competition used a scoring of "$\text{TruePositiveRate}\cdot\text{TrueNegativeRate}$".

Certainly this data set is not ideally suited for local SVMs: 
on one hand it certainly has higher dimensions than we would hope.
On the other hand the imbalance in the data has to be treated by hand.
For this we used the hinge loss with weight 0.987
after trying out different weights on a validation sample of size 100\;000.
For the training calculations, the use of the full 128GB memory per socket
allowed us to use (fine) cell sizes up to $100\;000$.
The splitting was not optimized and took about an hour.
Training and testing took 32.2 hours on our 11 worker machines,
see Table~\ref{table-ecbdl-result}.

Our off-the-cuff scores range from 0.422 to 0.456.
That would have landed us in the middle of the scores at the beginning of the competition.
Of the seven teams three had their first submissions between 0.3 and 0.42
and at some point found a way to boost it to 0.45 or more.
One of these, the team \textsc{HyperEns}, used
\emph{standard and budgeted SVMs with bayesian optimisation of parameters}
and started with submissions scoring around
0.34--0.38 and after ten days found a way to score over 0.45 and achieving
in the end 0.489 using \emph{4.7 days of parameter optimisation in a 16-core machine}.
The three best teams scored from start over 0.47.
The winning team \textsc{Efdamis} used feature weighting and random forests
and its best model took
\emph{39h of Wall-clock time in a 144-core cluster (not used exclusively)}
which is comparable to the training time of our best model.

Generally, we suppose that in such competitions the best results are achieved by
careful hand-optimisation of the features and the hyper-parameters.
We on the other hand aspire to realize fully automatic learning and hence
are not aiming to beat such results.

\begin{table}[hbtp]
  \caption{Results for the ECBDL'14 data set. Increasing the cell size
  achieved to increase the score = TPR$\cdot$TNR.}
  \label{table-ecbdl-result}
  \centering\footnotesize
  \begin{tabular}{rrrrrr}
    \toprule
    cell size & time & Score & TPR & TNR & work nodes
	\\
	\midrule
10\;000 & 4.8h & 0.422 & 0.656 & 0.643 & 7 \\
15\;000 & 7.2h & 0.433 & 0.664 & 0.652 & 7 \\
20\;000 & 9.3h & 0.438 & 0.664 & 0.660 & 7 \\
50\;000 & 14.0h & 0.453 & 0.666 & 0.679 & 11 \\
100\;000 & 32.2h & 0.456 & 0.667 & 0.680 & 11 \\
	\bottomrule
  \end{tabular}
\end{table}

\section{Conclusion}

The experiments show that on commodity hardware SVMs can be used
to train with millions of samples achieving at least decent errors in
a few hours--even including 5-fold cross validation on a $10\times10$ hyper-parameter grid. 
The results for ECBDL demonstrate that local SVMs scale trivially across clusters.
Hence, cloud scaling of fully automatic state-of the art SVMs is possible.
Together with the statistical guarantees in Theorem~\ref{Thm hinge rate}
it is clear that local SVMs provide a reliable and broadly usable machine learning system for large scale data.


\ifnum\statePaper=1
\subsubsection*{Acknowledgments}

We thank the Institute of Mathematics at the University of Zurich
for the computational resources.
\fi

\clearpage


\bibliographystyle{plainnat}

{
\bibliography{bib/additional,bib/steinwart-books,bib/steinwart-proc,bib/steinwart-article,bib/steinwart-mine,bib/christmann}}

\iftrue 

\clearpage
\appendix
\onecolumn
\section{Proofs}
\allowdisplaybreaks

A key concept to derive oracle inequalities and learning rates, which is used in the proof of Theorem \ref{corollar}, is the concept of entropy numbers, see \citet{CaSt90} or 
\citet[Definition~A.5.26]{StCh08}. Recall that, for normed spaces $(E,\|\,\cdot\,\|_E)$ and $(F,\|\,\cdot\,\|_F)$ as well as 
an integer $i\geq 1$, the $i$-th (dyadic) entropy number of a bounded, linear operator 
$S : E\to F$ is defined by
\begin{align*}
 e_i(S : E\to F) & := e_i(SB_E,\|\,\cdot\,\|_F) \\
 & := \inf\Biggl\{\e>0:\exists s_1,\ldots,s_{2^{i-1}}\in SB_E 
\text{ such that } SB_E \subset \bigcup_{j=1}^{2^{i-1}}(s_j + \e B_F)\Biggr\}\,,
\end{align*}
where we use the convention $\inf\emptyset:=\infty$, and $B_E$ as well as $B_F$ denote 
the closed unit balls in $E$ and $F$, respectively. 

\begin{proof}[Proof of Theorem \ref{corollar}]
We denote by $\tilde{H}$ the RKHS over $X$ with Gaussian kernel of width $\g_{\text{max}}$. Let $f_0 \in \tilde{H}$. Then, we can w.l.o.g.\ assume that $\|f_0\|_{\infty} \leq 1$, since the Gaussian kernel is bounded. For every $j\in \{1,\ldots m\}$ we define $f_j=\eins_{A_j}f_0=\widehat{f_{0_{|A_j}}}$ and remark that $f_{j_{|A_j}} \in H_{\g_{\text{max}}}(A_j)$ due to \citet[Exercise~4.4i)]{StCh08}. Hence, $f_j \in \hat{H}_{\g_{\text{max}}}$ by definition of $\hat{H}_{\g_{\text{max}}}$. Furthermore, since $\g_j \leq \g_{\text{max}}$ for every $j\in \{1,\ldots m\}$, \citet[Proposition~4.4.6]{StCh08} shows that $\hat{H}_{\g_{\text{max}}} \subset  \hat{H}_{\g_j}$ with
\begin{align}\label{norm_Hs}
\|f_j \|_{\hat{H}_{\g_j}} \leq \left( \frac{\g_{\text{max}}}{\g_j} \right)^{d/2} \|f_j \|_{\hat{H}_{\g_{\text{max}}}}.
\end{align}
Hence, we find that $f_j \in \hat{H}_{\g_{j}}$ for every $j \in \lbrace{1,\ldots,m \rbrace}$. Since
\begin{align*}
f_0=\sum_{i=1}^m f_j
\end{align*}
we conclude that $f_0 \in H$ by definition of $H$. Next, we observe with (\ref{norm_Hs}) that\begin{align*}
\sum_{j=1}^m \lb_j  \|\eins_{A_j}f_0\|^2_{\hat{H}_{\g_j}}= \sum_{j=1}^m \lb_j \|f_j\|^2_{\hat{H}_{\g_j}} \leq  \sum_{j=1}^m  \lb_j \left( \frac{\g_{\text{max}}}{\g_j} \right)^d \|f_j\|^2_{\tilde{H}} 
\leq \sum_{j=1}^m \lb_j \left( \frac{\g_{\text{max}}}{\g_j} \right)^d \| f_0\|^2_{\tilde{H}}.
\end{align*}
By using the latter inequality and the bound for the approximation error given in \citet[Theorem~8.18]{StCh08} with tail exponent $\t=\infty$ since $X$ is compact and with $\lb=\sum_{j=1}^m \lb_j \left( \frac{\g_{\text{max}}}{\g_j} \right)^d $, we find that
\begin{align}\label{approx_error}
\begin{split}
\sum_{j=1}^m \lb_j  \|\eins_{A_j}f_0\|^2_{\hat{H}_{\g_j}} + \RLP(f_0) - \RB &\Leq \sum_{j=1}^m \lb_j \left( \frac{\g_{\text{max}}}{\g_j} \right)^d \| f_0\|^2_{\tilde{H}}  + \RLP(f_0)-\RB \\
								     &\Leq \max \lbrace c_d,\tilde{c}_{d,\b}c_{\text{NE}} \rbrace \left( \sum_{j=1}^m \lb_j \left( \frac{\g_{\text{max}}}{\g_j} \right)^d \g_{\text{max}}^{-d}+ \g_{\text{max}}^{\b} \right)\\
								     &\Leq \hat{c}  \left( \sum_{j=1}^m \lb_j \g_j^{-d}+ \g_{\text{max}}^{\b} \right),
\end{split}
\end{align}
where $\hat{c}:= \max \lbrace c_d,\tilde{c}_{d,\b}c_{\text{NE}} \rbrace$ with $c_d, \tilde{c}_{d,\b}>0$. Next, \citet[Theorem 6]{EbStXXa} provides the bound $e_i(\mathrm{id} : H_{\g}(A_{j}) \to L_2(\TrP{A_j})) \leq  a_j i^{-\frac{1}{2p}}$ 
for $i\geq 1$ with $a_j = \tilde{c}_p \sqrt{P_X(A_{j})}\, r^\frac{d+2p}{2p} \g_j^{-\frac{d+2p}{2p}}$, where $\tilde{c}_p$ is a positive constant depending from $p$.
For the constant $a$ from Theorem \ref{main thm.} this yields
\begin{align*}
 & \Biggl(\max\Biggl\{ c_p m^{\frac{1}{2}} \Biggl(\sum_{j=1}^m \lb_j^{-p} a_j^{2p}\Biggr)^{\frac{1}{2p}},2\Biggr\}\Biggr)^{2p} \\
 & = \Biggl(\max\Biggl\{ c_p m^{\frac{1}{2}} \Biggl(\sum_{j=1}^m \lb_j^{-p} \left( \tilde{c}_p \sqrt{P_X(A_{j})}\, r^\frac{d+2p}{2p} \g_j^{-\frac{d+2p}{2p}} \right)^{2p}\Biggr)^{\frac{1}{2p}},2\Biggr\}\Biggr)^{2p} \\
 & =\Biggl(\max\Biggl\{ c_p \tilde{c}_p m^{\frac{1}{2}} r^\frac{d+2p}{2p} \Biggl(\sum_{j=1}^m \left( \lb_j^{-1}  \g_j^{-\frac{d+2p}{p}} P_X(A_{j})\right)^{p} \Biggr)^{\frac{1}{2p}},2\Biggr\}\Biggr)^{2p} \\
 &\leq \Biggl(\max\Biggl\{c_p \tilde{c}_p m^{\frac{1}{2p}} r^\frac{d+2p}{2p} \Biggl(\sum_{j=1}^m  \lb_j^{-1}  \g_j^{-\frac{d+2p}{p}} P_X(A_{j}) \Biggr)^{\frac{1}{2}},2\Biggr\}\Biggr)^{2p}\\
 & \leq \Biggl(\max\Biggl\{c_p \tilde{c}_p 16^{\frac{d}{2p}} r   \Biggl(\sum_{j=1}^m  \lb_j^{-1}  \g_j^{-\frac{d+2p}{p}} P_X(A_{j}) \Biggr)^{\frac{1}{2}},2\Biggr\}\Biggr)^{2p}\\
 &\leq C_p r^{2p} \Biggl(\sum_{j=1}^m  \lb_j^{-1}  \g_j^{-\frac{d+2p}{p}} P_X(A_{j}) \Biggr)^p + 4^{p}  \\
 & =: a^{2p}\,,
\end{align*}
where we used that $\| \cdot \|_p \leq m^{\frac{1-p}{p}} \| \cdot \|_1$ for $0<p<1$, as well as $m r^d\leq 16^d$ by \eqref{ex. Ueberdeckung} and that $C_p := c_p^{2p} \tilde{c}_p^{2p}16^d$.
Then, by using Theorem \ref{main thm.}, (\ref{approx_error}), the concavity of the function $t\mapsto t^{\frac{q+1}{q+2-p}}$ for $t\geq 0$ and the fact that $\tau \geq 1$ with $\t \leq n$ we obtain that
\begin{align*}
&\sum_{j=1}^m \lb_j \|f_{\D_{j},\lb_j,\g_j}\|^2_{\hat{H}_j}+\RLP(\fcl_{\D,\bs\lb,\bs\g})-\RB\\ 
  &\Leq  9 \left(\sum_{j=1}^m \lb_j  \|\eins_{A_j}f_0\|^2_{\hat{H}_{\g_j}}
  +\RLP(f_0) -\RB\right) 
  + C \left(a^{2p}n^{-1}\right)^{\frac{q+1}{q+2-p}}  
  + 3 \left(\frac{432 c_{\text{NE}}^{\frac{q}{q+1}} \t}{n}\right)^{\frac{q+1}{q+2}}  + \frac{30\t}{n}\\
  &\Leq 9\hat{c}  \left( \sum_{j=1}^m \lb_j \g_j^{-d}+ \g_{\text{max}}^{\b} \right) 
  + C \left(C_p r^{2p} \left(\sum_{j=1}^m  \lb_j^{-1}  \g_j^{-\frac{d+2p}{p}} P_X(A_{j}) \right)^pn^{-1} + 4^{p}n^{-1}\right)^{\frac{q+1}{q+2-p}} \\
  &\quad +  3 \left(\frac{432 c_{\text{NE}}^{\frac{q}{q+1}} \t}{n}\right)^{\frac{q+1}{q+2}} \! + \frac{30\t}{n}\\
  &\Leq  9\hat{c}  \left( \sum_{j=1}^m \lb_j \g_j^{-d}+ \g_{\text{max}}^{\b} \right) 
  + C \left(C_p r^{2p} \left(\sum_{j=1}^m  \lb_j^{-1}  \g_j^{-\frac{d+2p}{p}} P_X(A_{j}) \right)^pn^{-1} \right)^{\frac{q+1}{q+2-p}}\! \\ 
  &\quad + C\left(\frac{4^p\tau}{n} \right)^{\frac{q+1}{q+2-p}} + 3 \left(\frac{432 c_{\text{NE}}^{\frac{q}{q+1}} \t}{n}\right)^{\frac{q+1}{q+2}}  + \frac{30\t}{n}\\
  &\Leq \tilde{C}_{\b,d,p,q} \Biggl(   \sum_{j=1}^m \lb_j \g_j^{-d}+ \g_{\text{max}}^{\b}  +\left( r^{2p} \left(\sum_{j=1}^m  \lb_j^{-1}  \g_j^{-\frac{d+2p}{p}} P_X(A_{j}) \right)^p n^{-1} \right)^{\frac{q+1}{q+2-p}} \\
  &\quad + \left(\frac{\tau}{n} \right)^{\frac{q+1}{q+2-p}} +  \left(\frac{\t}{n}\right)^{\frac{q+1}{q+2}}  + \frac{\t}{n}  \Biggr)\\
  &\Leq C_{\b,d,p,q}  \Biggl(   \sum_{j=1}^m \lb_j \g_j^{-d}+ \g_{\text{max}}^{\b}   +\left( r^{2p} \left(\sum_{j=1}^m  \lb_j^{-1}  \g_j^{-\frac{d+2p}{p}} P_X(A_{j}) \right)^p n^{-1} \right)^{\frac{q+1}{q+2-p}} +\left(\frac{\t}{n}\right)^{\frac{q+1}{q+2}} \Biggr)\\
\end{align*}
holds with probability $P^n$ not less than $1-3e^{-\t}$, where the constants $\tilde{C}_{\b,d,p,q}$ and $C_{\b,d,p,q}$ are given by $\tilde{C}_{\b,d,p,q} :=\max \lbrace 9\hat{c}, C \cdot C_p^{\frac{q+1}{q+2-p}}, C \cdot 4^{\frac{p(q+1)}{q+2-p}},3 c_{\text{NE}}^{\frac{q}{q+2}}\cdot 432^{\frac{q+1}{q+2}},30  \rbrace$ and $C_{\b,d,p,q} :=\max \lbrace 9\hat{c}, C \cdot C_p^{\frac{q+1}{q+2-p}}, 3 \cdot C \cdot 4^{\frac{p(q+1)}{q+2-p}},9 c_{\text{NE}}^{\frac{q}{q+2}}\cdot 432^{\frac{q+1}{q+2}},90  \rbrace$.
\end{proof}

\begin{proof}[Proof of Theorem \ref{Thm hinge rate}]
First we simplify the presentation by using the sequences $\tilde{\lb}_{n}:=c_2n^{-(\b+d)\k}$ and $\tilde{\g}_{n}:=c_3n^{-\k}$ with $\k:=\frac{(q+1)}{\b(q+2)+d(q+1)}$. Then, we find with Theorem \ref{corollar} together with $r_n = c_1 n^{-\nu}, \lb_{n,j} = r_n^d  \tilde{\lb}_{n}$ and $\g_{n,j} = \tilde{\g}_{n}$ and with $\sum_{j=1}^{m_n} P_X(A_{j})=1$ and $m_n \leq 16^d r_n^{-d}$ that
\begin{align*}
&\RLP(\fcl_{\D,\bs\lb,\bs\g})-\RB \nonumber\\ 
&\Leq C_{\b,d,p,q}  \Biggl(  \sum_{j=1}^{m_n} \lb_{n,j} \g_{n,j}^{-d}+ \g_{\text{max}}^{\b} +\left( r_n^{2p} \left(\sum_{j=1}^{m_n}  \lb_{n,j}^{-1}  \g_{n,j}^{-\frac{d+2p}{p}} P_X(A_{j}) \right)^p n^{-1} \right)^{\frac{q+1}{q+2-p}} + \left(\frac{\t}{n}\right)^{\frac{q+1}{q+2}}  \Biggr)\\
&\,=\, C_{\b,d,p,q}  \Biggl(  m_n r_n^d \tilde{\lb}_{n} \tilde{\g}_{n}^{-d}+ \tilde{\g}_{n}^{\b}  +\left( r_n^{2p} \left(r_n^{-d} \tilde{\lb}_{n}^{-1}  \tilde{\g}_{n}^{-\frac{d+2p}{p}}  \sum_{j=1}^{m_n}  P_X(A_{j}) \right)^p n^{-1} \right)^{\frac{q+1}{q+2-p}} + \left(\frac{\t}{n}\right)^{\frac{q+1}{q+2}}  \Biggr)\\
&\Leq 16^d C_{\b,d,p,q}  \Biggl(  \tilde{\lb}_{n} \tilde{\g}_{n}^{-d}+ \tilde{\g}_{n}^{\b}  +\left( r_n^{p(2-d)}  \tilde{\lb}_{n}^{-p}  \tilde{\g}_{n}^{-(d+2p)} n^{-1} \right)^{\frac{q+1}{q+2-p}} + \left(\frac{\t}{n}\right)^{\frac{q+1}{q+2}}  \Biggr)\\
&\Leq C_{\b,\nu,d,p,q} \Bigg( n^{-(\b+d)\k} n^{+d\k} +  n^{-\b\k} + \left(\frac{n^{-p\nu(2-d)}}{n^{-p(\b+d)\k} n^{-(d+2p)\k+1}}\right)^{\frac{q+1}{q+2-p}} +  \left( \frac{\t}{n}\right)^{\frac{q+1}{q+2}} \Bigg)\\
 &\,=\, C_{\b,\nu,d,p,q}\Bigg( 2n^{-\b\k} + \left(\frac{n^{-p\left[\nu(2-d)-(\b+d)\k-2\k\right]}}{ n^{-d\k+1}}\right)^{\frac{q+1}{q+2-p}} 
  + \left( \frac{\t}{n}\right)^{\frac{q+1}{q+2}} \Bigg)\\
 &\,=\,  C_{\b,\nu,d,p,q} \Bigg(  2n^{-\b\k} + \frac{n^{-\frac{p(q+1)}{q+2-p}\left[\nu(2-d)-(\b+d+2)\k\right]}}{ n^{\frac{\b(q+2)(q+1)}{(\b(q+2)+d(q+1))(q+2-p)}}} 
  +   \left( \frac{\t}{n}\right)^{\frac{q+1}{q+2}} \Bigg)\\
   &\Leq C_{\b,\nu,d,p,q} \Bigg(  2n^{-\b\k} + \frac{n^{-\frac{p(q+1)}{q+2-p}\left[\nu(2-d)-(\b+d+2)\k\right]}}{ n^{\frac{\b(q+1)}{(\b(q+2)+d(q+1))}}} 
  +   \left( \frac{\t}{n}\right)^{\frac{q+1}{q+2}} \Bigg)\\
  &\Leq C_{\b,\nu,d,p,q}  \Bigg( 2n^{-\frac{\b(q+1)}{\b(q+2)+d(q+1)}} +n^{-\frac{\b(q+1)}{\b(q+2)+d(q+1)}} n^{-\frac{p(q+1)}{q+2}\left[\nu(2-d)-\frac{(\b+d+2)(q+1)}{\b(q+2)+d(q+1)}\right]} \!
  +  \left( \t n^{-1}\right)^{\frac{q+1}{q+2}} \Bigg)\\
  &\Leq C_{\b,\nu,\xi,d,q} \tau^{\frac{q+1}{q+2}} \cdot  n^{-\frac{\b(q+1)}{\b(q+2)+d(q+1)}+\xi}\\
\end{align*}
holds with probability $P^n$ not less than $1-3e^{-\t}$, where the constants $C_{\b,\nu,d,p,q},c_1,c_2,c_3 >0$ depend on $\b,\nu,d,p,q$ and the constant $C_{\b,\nu,\xi,d,q}>0$ depends on $\b,\nu,\xi,d,q$. Furthermore, we remark that we chose $p$ sufficiently close to zero such that $\xi \geq \frac{p(q+1)}{q+2}\left(\nu(d-2)+\frac{(\b+d+2)(q+1)}{\b(q+2)+d(q+1)}\right) \geq 0$. 
\end{proof}

\section{Appendix}

\begin{theorem}\label{main thm.}
Let $P$ be a distribution on $X \times Y$ with noise exponent $q \in (0,\infty]$ and let $L: Y\times\R\to [0,\infty]$ be the hinge loss. Furthermore let (A) be satisfied with $H_j:=H_{\g_j}(A_j)$ and assume that, for fixed $n\geq 1$, there exist constants $p\in(0,1)$ and $a_1,\ldots,a_m > 0$ 
such that for all $j\in \{1,\ldots,m\}$
\begin{align} \label{entropy assumption}
 e_i(\mathrm{id} : H_j \to L_2(\TrP{A_j})) \leq a_j\, i^{-\frac{1}{2p}} 
 \,,\qquad\qquad i\geq1\,. 
\end{align}
Finally, fix an $f_0\in H$. Then, for all fixed $\t>0$, 
$\bs\lb=(\lb_1,\ldots,\lb_m)>0$, $\bs\g=(\g_1,\ldots,\g_m)>0$ and 
\begin{align*}
  a := \max\left\{c_p m^{\frac{1}{2}} \left(\sum_{j=1}^m \lb_j^{-p} a_j^{2p}\right)^{\frac{1}{2p}}  , 2\right\}\,
\end{align*}
the VP-SVM given by \eqref{VP-SVM} satisfies
\begin{align*}
  &\sum_{j=1}^m \lb_j \|f_{\D_{j},\lb_j,\g_j}\|^2_{\hat{H}_j}+\RLP(\fcl_{\D,\bs\lb,\bs\g})-\RB \nonumber\\ 
  &\Leq  9 \!\left(\sum_{j=1}^m \lb_j \|\eins_{A_j}f_0\|^2_{\hat{H}_j}
  +\RLP(f_0) -\!\RB\!\right) 
  + C \left(a^{2p}n^{-1}\right)^{\frac{q+1}{q+2-p}} 
  + 3 \left(\frac{432 c_{\text{NE}}^{\frac{q}{q+1}} \t}{n}\right)^{\frac{q+1}{q+2}}  + \frac{30\t}{n}
\end{align*}
with probability $\P^n$ not less than $1-3e^{-\t}$, where 
$C>0$ is a constant only depending on $p$.
\end{theorem}

\begin{proof}[Proof of Theorem \ref{main thm.}]
One can obtain the result directly by an application of \citet[Theorem~5]{EbStXXa}. To this end, we note that the hinge loss is Lipschitz continuous and can be clipped at $M=1$. 
Since  $H$ is the sum-RKHS of RKHSs with Gaussian kernels and the Gaussian kernel is bounded, w.l.o.g.\ we assume for $f_0 \in H$ that $\|f_0\|_{\infty} \leq 1$. Hence, $\|L \circ f_0\|_\infty\leq 2$ and therefore 
$B_0=2$. Furthermore \citet[Theorem~8.24]{StCh08} showed that for the hinge loss the constants $V$ and $\vt$ from \citet[Theorem~5]{EbStXXa} can be achieved by $V=6c_{\text{NE}}^{\frac{q}{q+1}}$ and $\vt=\frac{q}{q+1}$.
That means
\begin{align*}
  & \sum_{j=1}^m \lb_j \|f_{\D_{j},\lb_j,\g_j}\|^2_{\hat{H}_j}+\RLP(\fcl_{\D,\bs\lb,\bs\g})-\RB\\
  &\Leq  9 \left(\sum_{j=1}^m  \lb_j\|\eins_{A_j} f_0\|^2_{\hat{H}_j}
  +\RLP(f_0) -\RB \right) 
  + C \left( \frac{a^{2p}}{n}\right)^{\frac{1}{2-p-\vt+\vt p}} 
  + 3 \left(\frac{72 V\t}{n}\right)^{\frac{1}{2-\vt}}  + \frac{15B_0\t}{n}\\
  &\,=\, 9 \left( \sum_{j=1}^m   \lb_j \|\eins_{A_j} f_0\|^2_{\hat{H}_j}
  +\RLP(f_0) -\RB \right) 
  +C\left(\frac{a^{2p}}{n}\right)^{\frac{q+1}{q+2-p}} 
  + 3 \left(\frac{432 c_{\text{NE}}^{\frac{q}{q+1}} \t}{n}\right)^{\frac{q+1}{q+2}}  + \frac{30\t}{n}\
\end{align*}
holds with probability $\P^n$ not less than $1-3e^{-\t}$.
\end{proof}


\section{Some more details for results}

In this section we give some more technical details
the pseudo-code for local SVMs (Algorithm~\ref{algo:locSVM}),
and some more results of the experiments.
Firstly, some more details on how the experiments were performed:
\begin{description}
\item[Hyperparameter grid]
  We used \texttt{liquidSVM}'s default grid:
  the $\lambda$ are geometrically spaced between $0.01/{\tilde n}$ and $0.001/\tilde n$ where $\tilde n$ is the number of samples contained in the $k-1$ folds currently used for training.
  The $\gamma$ are geometrically spaced between $5r$ and $0.2r \tilde n^{-1/d}$ where $r$ is the radius of the cell, $d$ is the dimension of the data and $\tilde n$ is as above.
\item[Spatial partitioning scheme] The segmentation for mid-sized data sets $n\le 50000$ finds centers for the Voronoi cells using the farthest-first-traversal algorithm on the entire data set. For larger data sets a random subsample of the full data set is created and the splitting described above is applied recursively.

In \texttt{liquidSVM}, this is achieved with value \texttt{6} for the \texttt{partition} argument to \texttt{scripts/mc-svm.sh} (or the \texttt{-P 6} mode of \texttt{svm-train}).
\item[Random Chunks scheme]
The data is split into random partitions of the specified size.
In \texttt{liquidSVM}, this is achieved with value \texttt{1} for the \texttt{partition} argument to \texttt{scripts/mc-svm.sh} (or the \texttt{-P 1} mode of \texttt{svm-train}).
\end{description}

\begin{algorithm}[t]
\caption{Local SVM training and testing.}
\label{algo:locSVM}
\begin{algorithmic}[1]
{\footnotesize
\REQUIRE A training dataset $D$, split into cells $D_1,\ldots,D_m$, $m\ge1$,
a set $\Gamma\subset \R_{>0}$ of $\gamma$-candidates,
a set $\Lambda\subset \R_{>0}$ of $\lambda$-candidates,
the number of folds $k$ for cross-validation,
and a test set $D^T$, split into cells $D_1^T,\ldots,D_m^T$.
\ENSURE Test error
\FORALL{$j=1,\dots,m$}
	\STATE Split the cell $D_j$ into $k$ random parts $D_{j,1},\ldots,D_{j,k}$.
	\FORALL{$\ell=1,\ldots,k$}
		\STATE $D'_{j,\ell} := D_j\setminus D_{j,\ell}$
		\STATE cache pre-kernel matrix $(x_1,x_2)\to \|x_1-x_2\|_2$ for $x_1,x_2$ in $D'_{j,\ell}$
		\FORALL{$\gamma\in\Gamma$}
			\STATE use cached pre-kernel matrix to calculate kernel matrix with bandwidth $\gamma$
			\FORALL{$\lambda\in\Lambda$}
        		\STATE Train an SVM $f_{D'_{j,\ell},\lambda,\gamma}$ of the form \eqref{VP-SVM} (possibly using as warm-start the solution for the previous $\lambda$-candidate).
        		\STATE Calculate and save the validation risk $\mathcal{R}_{L,D_{j,\ell}}(f_{D'_{j,\ell},\lambda,\gamma})$
        	\ENDFOR
			\STATE Let $f_{D_j,\lambda,\gamma}$ be the linear combination of the $(f_{D'_{j,\ell},\lambda,\gamma})_{1\le \ell\le k}$ with weights exponential in
			$\mathcal{R}_{L,D_{j,\ell}}(f_{D'_{j,\ell},\lambda,\gamma})$.
        	\STATE Save the validation risk $\mathcal{R}_{L,D_{j}}(f_{D_{j},\lambda,\gamma})$
        \ENDFOR
    \ENDFOR
\ENDFOR
\FORALL{$j=1,\dots,m$}
	\STATE Select the $\gamma_j,\lambda_j$-combination minimizing the combined validation risk.
\ENDFOR
\FORALL{$j=1,\dots,m$}
	\STATE Calculate test error $\mathcal{R}_{L,D^T_{j}}(f_{D_{j},\lambda_j,\gamma_j})$
	on test cell $D_j^T$.
\ENDFOR
\STATE \textbf{return} global test error $\frac1{|D^T|}\sum_{j=1}^m |D^T_j|
\cdot\mathcal{R}_{L,D^T_{j}}(f_{D_{j},\lambda_j,\gamma_j})$.
}
\end{algorithmic}
\end{algorithm}

\begin{table}[hbt]
  \caption{Training times divided by the product of training set size, cell size and dimension (in seconds times $10^9$).
  This shows that the time complexity in \eqref{eq:complexity}
  is fulfilled quite nicely.
  }
  \label{time-table-O}
  \centering\footnotesize
  \begin{tabular}{lrrrr}
    \toprule
& 2000& 5000& 10000& 15000
\\
\midrule
&\multicolumn{4}{c}{training time} \\ \cmidrule{2-5}
\sc higgs & 33 & 29 & 29 & 29 \\
\sc hepmass & 22 & 19 & 19 & 20 \\
\sc gassensor & 20 & 19 & 19 & 19 \\
\sc susy & 45 & 39 & 38 & 38 \\
\sc covtype & 10 & 9 & 9 & 9 \\
\sc cod-rna & 52 & 51 & 45 & 48 \\
\sc skin & 116 & 117 & 106 & 111 \\
\bottomrule
  \end{tabular}
\end{table}

\begin{figure}[hbt]
  \centering
  \includegraphics[width=\textwidth]{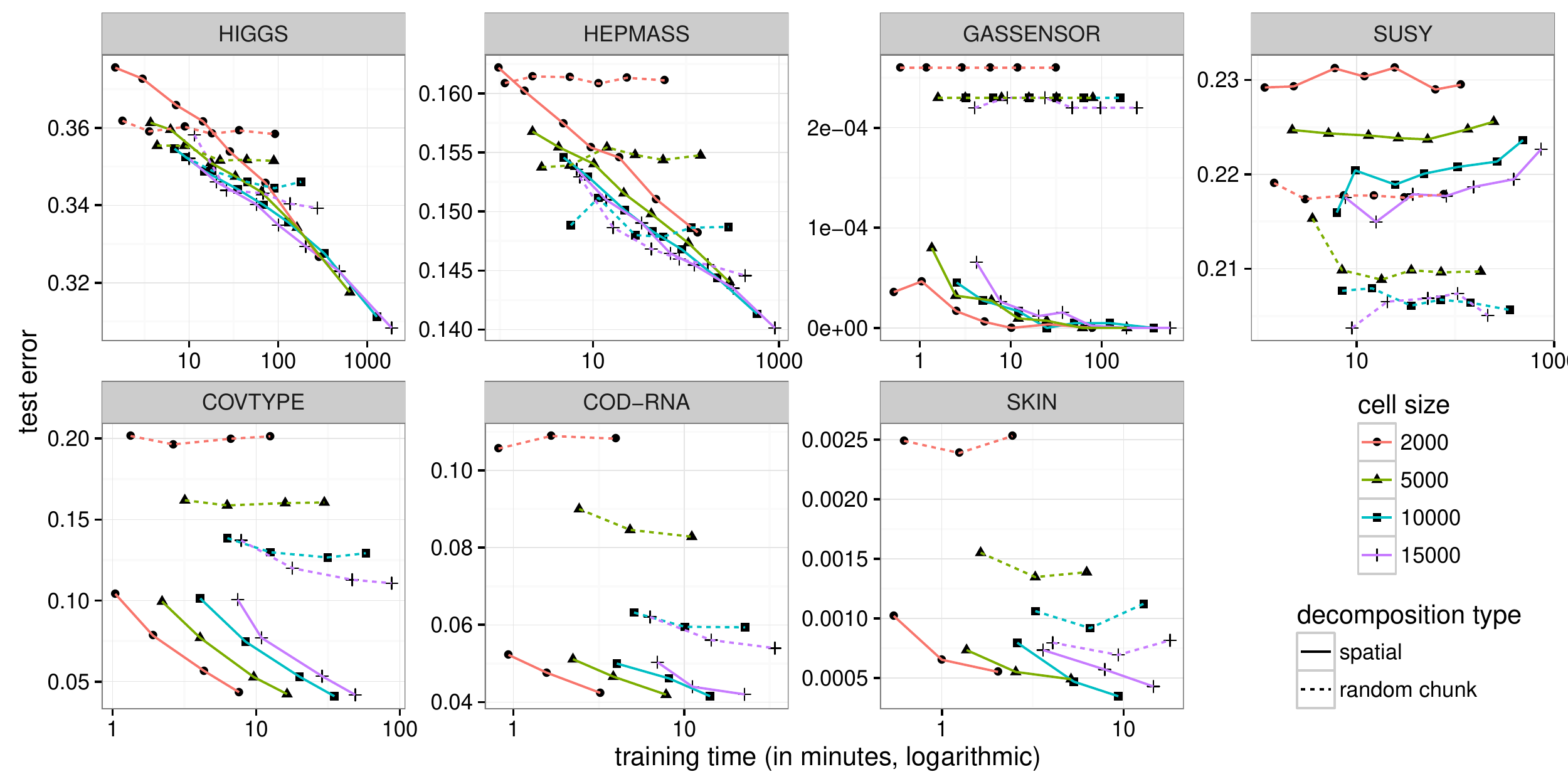}
  \caption{Training time vs. test error, this is the final trade-off.\label{figure:time-error}}
\end{figure}

\fi
\end{document}